%% file: 0_main.tex
\documentclass[runningheads]{llncs}
\usepackage[T1]{fontenc}
\usepackage{amsmath,amsfonts}
\usepackage{graphicx}  % Include graphics
\usepackage{xcolor}    % Colored text
\usepackage{subcaption}
\usepackage{wrapfig}
\usepackage{graphicx}
\usepackage{cite}
\makeatletter
\renewcommand\section{\@startsection{section}{1}{\z@}%
                       {12\p@ \@plus 4\p@ \@minus 4\p@}%
                       {8\p@ \@plus 4\p@ \@minus 4\p@}%
                       {\normalfont\large\bfseries\boldmath
                        \rightskip=\z@ \@plus 8em\pretolerance=10000 }}
\renewcommand\subsection{\@startsection{subsection}{2}{\z@}%
                       {8\p@ \@plus 4\p@ \@minus 4\p@}%
                       {6\p@ \@plus 4\p@ \@minus 4\p@}%
                       {\normalfont\normalsize\bfseries\boldmath
                        \rightskip=\z@ \@plus 8em\pretolerance=10000 }}
\renewcommand\subsubsection{\@startsection{subsubsection}{3}{\z@}%
                       {4 \p@ \@plus 4\p@ \@minus 4\p@}%
{-0.5em \@plus -0.22em \@minus -0.1em}%
                       {\normalfont\normalsize\bfseries\boldmath
                        \rightskip=\z@ \@plus 8em\pretolerance=10000 }}
\renewcommand\paragraph{\@startsection{paragraph}{4}{\z@}%
                       {-12\p@ \@plus -4\p@ \@minus -4\p@}%
                       {-0.5em \@plus -0.22em \@minus -0.1em}%
                       {\normalfont\normalsize\itshape}}
\makeatother

\begin{document}
\title{Group-Control Motion Planning Framework for Microrobot Swarms in a Global Field}
\titlerunning{Group-Control for Microrobot Swarms}
\author{Siyu Li\inst{1}
\and
Afagh Mehri Shervedani\inst{1} 
\and
Milo\v s \v Zefran\inst{1}
\and
Igor Paprotny \inst{1}}
%
% \orcidID{0009-0007-9554-0681}
\authorrunning{S. Li et al.}
\institute{University of Illinois Chicago, Chicago IL 60607, USA 
\email{\{sli230,amehri2,mzefran,paprotny\}@uic.edu}}
\maketitle              % typeset the header of the contribution
\begin{abstract}
This paper investigates how a novel paradigm called group-control can be effectively used for motion planning for microrobot swarms in a global field. We prove that Small-Time Local Controllability (STLC) in robot positions is achievable through group-control, with the minimum number of groups required for STLC being $\log_2(n + 2) + 1$ for $n$ robots. We then discuss the trade-offs between control complexity, motion planning complexity, and motion efficiency. We show how motion planning can be simplified if appropriate primitives can be achieved through more complex control actions. We identify motion planning problems that balance the number of motion primitives with planning complexity. Various instantiations of these motion planning problems are explored, with simulations used to demonstrate the effectiveness of group-control.

\end{abstract}
%
%
%

% \section{Table of contents}

% \begin{enumerate}
%     \item Intro/literature review
    
%     Introduce motion planning for a group of robots with a global control signal.
%     \item Background: Model, Embedding, PFSM
%     \item Group-based control

%     Introduce the result that the minimum number of groups for STLC is $\log_2 (n+2) + 1$. 
%     \item STLC

%     Start by discussing the structure of the vector fields and their Lie brackets.

%     \item Motion planning

%     Start by introducing subgroups. Note that instead of subgroups -- which are logical -- you can also introduce subgroups that are physical (by increasing the number of groups).

%     Motion planning curve
    
%     Complexity hierarchy
    
%     Motion planning cases
%     \begin{enumerate}
%         \item Extreme 1: Full motion planning, elementary control

%         \item Extreme 2: Full control, no motion planning

%         \item Interesting intermediate cases
%     \end{enumerate}
    
%     \item Conclusion and future work
% \end{enumerate}

\input{1_Introduction}
\input{2_Background}

\input{3_Group_based_control}

\input{4_STLC_proof}

\input{5_Motion_planning}

\input{6_Conclusion}

%\bibliographystyle{splncs04}
%\bibliography{references}
\newpage
\bibliographystyle{splncs04}
\bibliography{references,additional}

% \section{Numerical Simulations}
% \subsection{Numerical control results}
% \subsection{Motion planning}
% \SL{Add results for various order of primitives}

\end{document}

%% file: 1_Introduction.tex
\section{Introduction}

Microrobots have gained significant attention for their potential in medical, environmental, and industrial applications. Effective control mechanisms are crucial to enable their diverse functionalities. Various control modalities have been demonstrated for microrobots, including magnetic control~\cite{pawashe2009modeling,  qiu2015magnetic, li_development_2018}, optical control~\cite{hu2011micro, banerjee2012real, li2016swimming, palagi2019light}, and acoustic control~\cite{jeong2020acoustic, aghakhani2022high}. These systems are controlled by global fields, and as a result, they are massively underactuated.

Currently, parallel control of multiple microrobots in a global field is based on the design of robots with distinct physical characteristics, resulting in different responses of each robot to the same control signal. The Global Control Selective Response (GCSR) paradigm~\cite{donald_planar_2008,donald2013planning} uses fabrication to make each robot respond appropriately to the control signal. Turning-rate Selective Control (TSC)~\cite{paprotny_turning-rate_2013} is a variation of GCSR and differentiates robots by explicitly designing them with different turning rates. These methods scale poorly to larger swarms due to fabrication complexities.  Ensemble Control (EC)~\cite{becker2014controlling} leverages differences in the linear velocity and turning rate parameters among robots stemming from randomness in fabrication to control the position of individual robots. Unfortunately, the ability to control individual robots is inherently limited by noise, so this approach also scales poorly.

This paper focuses on electrostatic stress-engineered MEMS microrobots (MicroStressBot), originally described in~\cite{donald_untethered_2006}. The swarm is powered by a uniform electrostatic field generated by a substrate on which the robots move, which means that all robots are controlled by a single global signal. Their size limits onboard control logic and power storage, making it essential to simultaneously coordinate large groups of microrobots for applications like micro-assembly or drug delivery. Our work explores the concept of onboard multi-stage Physical Finite-State Machines (PFSMs) introduced in~\cite{paprotny2017finite}. PFSMs allow robots to be individually addressed and activated one by one. Our previous work~\cite{li2022group} takes the idea of PFSMs further by introducing \emph{group-based control}, where several robots can be activated together as a group. This paper investigates the trade-off between the motion planning complexity, control complexity, and motion efficiency within the {group-control} framework. In particular, we introduce several control and motion planning problems and discuss their scalability.
 
Motion planning for microrobots involves determining a sequence of movements to reach a given configuration while avoiding obstacles and collisions. Various methods exist, including graph-based methods like Dijkstra's and $A^*$ search~\cite{lavalle2006planning}, sampling-based algorithms like Rapidly-exploring Random Trees (RRTs)~\cite{lavalle_randomized_2001} and optimization-based methods~\cite{karaman2011sampling,zucker2013chomp}. Recent advances in machine learning, such as~\cite{qureshi2020motion, strudel2021learning, wang2020neural}, leverage neural networks to speed up RRTs and control techniques for mobile robot navigation. However, despite a wealth of solutions for motion planning problems, motion planning for microrobot swarms in a global field remains a challenging problem because the system is high-dimensional yet massively underactuated, having a single control signal.
%We thus explore the motion planning procedure for \emph{Group-Control} with the emphasis on the geometry of the problem.

Fundamentally, when a global control signal controls a microrobot swarm, the robots can be controlled individually only when each robot responds to the control signal differently. In this paper, we introduce the \emph{group-control} framework to differentiate the motion of each microrobot in the swarm. We prove that Small-Time Local Controllability (STLC) is achievable in robot positions given an appropriate group allocation, with the minimum number of groups required for STLC being $log_2(n+2)+1$ for $n$ robots. Additionally, we discuss the complexity trade-offs between control, motion planning, and motion efficiency. We identify motion planning problems that balance the number of robot groups and motion primitives with planning complexity. Various instantiations of these motion planning problems are explored, and simulations are provided that demonstrate the scalability and effectiveness of the proposed algorithms.

% \cite{wei2017stimuli, fath2022recent, yao_directed_2020}.
% highlighting their effectiveness in high-dimensional and complex environments in~\cite{karaman2011sampling}.  

%In medical applications, the safety and reliability of microrobots are paramount. Effective control mechanisms ensure microrobots operate within safe parameters, reducing the risk of patient harm. 

%Microrobots hold promise for a variety of applications, including minimally invasive surgery, targeted drug delivery, environmental monitoring, and microassembly. However, challenges such as power supply, communication, and precise control in complex environments remain. 

%% file: 2_Background.tex
\section{Background}

\subsection{MicroStressBot}

The MicroStressBot is a 120 $\mu$m $\times$ 60 $\mu$m $\times$ 10 $\mu$m mobile microrobot platform introduced in~\cite{donald_untethered_2006}. A MicroStressBot has two actuated internal degrees of freedom (DOF): an untethered scratch-drive actuator (USDA) that provides forward motion and a steering-arm actuator that determines whether the robot moves in a straight line or rotates (Fig.~\ref{fig:stressbot}). A single MicroStressBot can have its arm raised or lowered, depending on the voltage applied across a substrate formed by an electrode array.  When a sufficiently high voltage is applied to the substrate, the arm is pulled into contact with the substrate, and the robot rotates around the contact point. In contrast, when the voltage is reduced below a threshold, the arm is raised, and the robot moves forward.

MicroStressBot control has been successfully implemented in~\cite{donald2013planning}. It has been shown that if the pull-down and release voltages of the robots are different, they can be independently controlled. However, it is difficult to consistently fabricate robots to respond in the desired way. As a result, the approach scales poorly.

\begin{wrapfigure}{r}{0.5 \textwidth}
\vspace{-3mm}
    \centering
    \includegraphics[width=0.5 \textwidth]{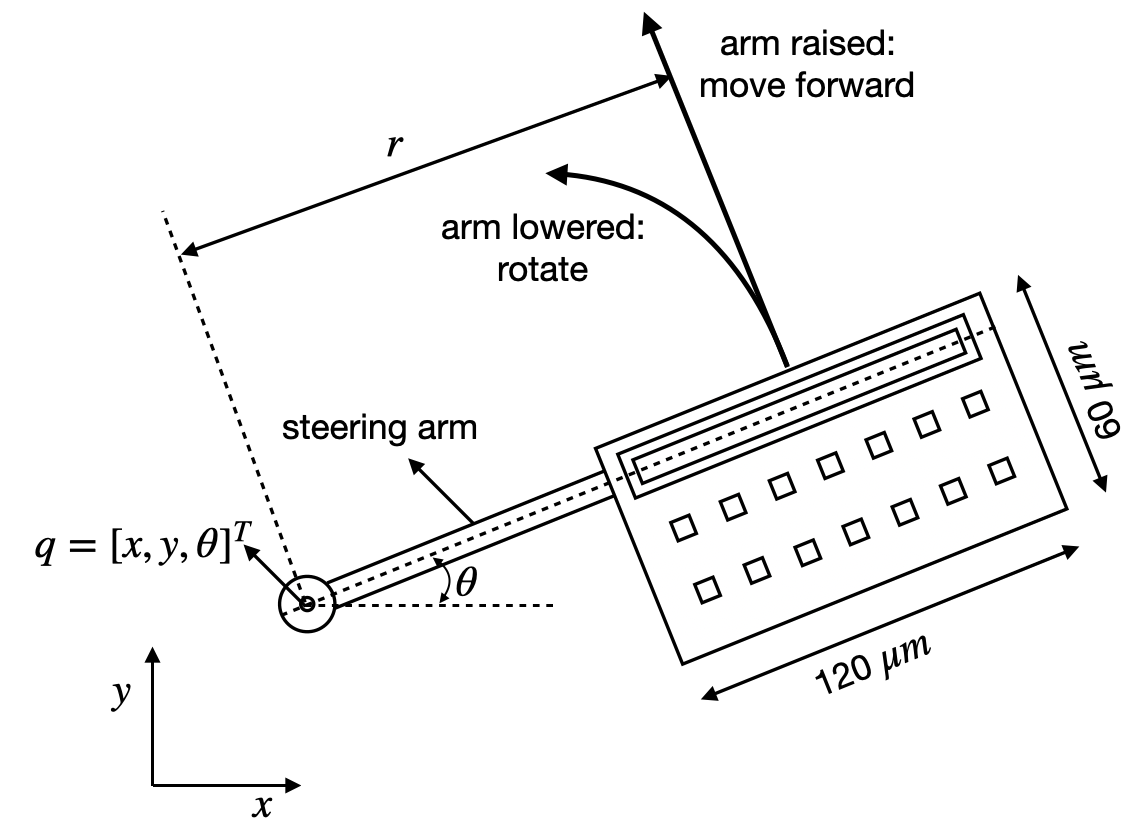}
    \caption{The schematic of a MicroStressBot.}
    \label{fig:stressbot}
    \vspace{-10mm}
\end{wrapfigure}

\subsection{Physical FSM Robots}

One solution to dramatically improve the scalability of the microrobot swarm control is to make the robots respond to a temporal sequence of (a small number of) voltage levels rather than to the voltage directly. Finite-State Machines (FSMs) can accept a set of input sequences~\cite{sipser1996introduction} (sequences of control signal levels). Previously, in~\cite{paprotny2017finite}, we proposed the on-board MEMS Physical FSM (PFSM) that, upon acceptance of a unique control signal sequence, causes the arm of the MicroStressBot to be pulled up or down. PFSMs can be constructed from several basic modules that are combined together and thus fabricated efficiently. In this work, we build on this idea to propose \emph{group-control}. In particular, we use the fact that several PFSM modules, each corresponding to one group, can be combined together so that each robot can belong to several groups.

\subsection{Group-Control}\label{group_allo}

%A common way to control multiple robots independently with a single global input is to design robots to be physically distinct so that robots can behave differently according to the same input. However, it is difficult and time-consuming to do so when the number of robots is large.

The core idea of {group-control} is that by equipping the robots with PFSMs, they can be assigned to (several) different groups. When the activation sequence corresponding to a group is sent to the swarm, all the robots that belong to the group are activated. In this paper, activating a MicroStressBot corresponds to raising its arm so that it can move forward (translate); otherwise, the robot rotates. By assigning each robot to a unique subset of groups, we can differentiate between the robots and make them respond differently to a sequence of inputs that activate all the groups one after the other, each time moving the robots belonging to the group for some distance.

Table~\ref{Tab:Groups} shows how six robots can be assigned to different groups. To make the selection of groups unique to each robot, if we have $n$ robots, we need $m=\log_2 (n+2)+1$ groups ($n+2$ rather than $n$ because each robot needs to belong to at least one group so it can move forward, and no robot can belong to all the groups so it can rotate). We add one more group where none of the robots translate, they all rotate in place. This special group will be the group $G_m$.

As can be easily seen from Table~\ref{Tab:Groups}, group assignment corresponds to assigning a unique bit pattern (a subset of groups) to each robot, where each bit corresponds to a group (labeled $G_i$, $i=1, 2, 3$ in Table~\ref{Tab:Groups}). The robot then belongs to the groups where the corresponding bit equals $1$. For example, robot $R_1$ belongs to the group $G_3$, robot $R_3$ belongs to the groups $G_2$ and $G_3$, etc. The group $G_4$ is the additional group with all zeros where all the robots are rotating. It is clear that this method guarantees that each robot has a distinct group allocation. Also, note that when $n=2^{m-1}-2$ (the maximum number of robots that can be controlled by $m$ groups), each group (apart from the special group $G_m$) contains exactly half, that is, $\frac{n}{2}=2^{m-2}-1$, robots.

\begin{wraptable}{l}{0.4 \textwidth}
\vspace{-9mm}
\begin{center}
 \begin{tabular}{||c | c  c  c c c c||} 
 \hline
 & \multicolumn{6}{c ||}{Robot} \\
 Group & $R_1$ & $R_2$ & $R_3$ & $R_4$ & $R_5$ & $R_6$\\ [0.5ex] 
 \hline
 $G_1$ & 0 & 0 & 0 & 1 & 1 & 1 \\ 
 \hline
 $G_2$ & 0 & 1 & 1 & 0 & 0 & 1 \\
 \hline
 $G_3$ & 1 & 0 & 1 & 0 & 1 & 0\\ \hline
 \hline
$G_4$ & 0 & 0 & 0 & 0 & 0 & 0\\
 \hline
\end{tabular}
\end{center}
\vspace{-3mm}
\caption{Allocating 6 robots to 3 groups.}
\label{Tab:Groups}
\vspace{-5mm}
\end{wraptable}

The group allocation can be realized by the on-board PFSMs. At each time step $k$, only one group of robots is activated, and the robots belonging to that group translate (move forward) while the remaining robots rotate. We call a sequence of group selections an \emph{activation sequence}.
% Formally, it is defined as:
% \begin{definition}
%     (Activation signal or switching signal) For a set of groups $\mathcal{G}$ with $n$ robots, where $\mathcal{G}= \{ G_1, G_2,\ldots, G_m\}$. There is a bijection  $\mathcal{G} \to \mathcal{A}$,  where $\mathcal{A} = \{a^1,\ldots,a^m\}$ and $a^i$ is a binary vector, i.e. $a^i = [a^i_1,\ldots,a^i_n] \in \{0,1\}^n$. We call $a^i$ as the activation signal.
% \end{definition}
% \begin{definition}
%     (Activation sequence) $Act$ is a $k$ length sequence $a(1),\ldots,a(i),\ldots,a(k)$, where $a(i)$ is an activation signal.
% \end{definition}
If $m$ is the number of groups and $n$ is the number of robots, we would need $k = O(\log_2 m)=O(\log_2(\log_2(n+2)+1))$ bits to select a group through a PFSM~\cite{paprotny2017finite}. Thus, PFSMs for group-based control can be significantly simpler than when each robot needs to be selected individually.

\section{Modeling}

\subsection{Dynamics}

To describe {group-control} mathematically, we start with a dynamic model of a MicroStressBot. The robot can move freely on a horizontal plane, so its configuration space is the Euclidean group $SE(2)$. We will describe the state of the robot $i$ with a vector $[x_i, y_i, \theta_i]^T$, where $x_i$ and $y_i$ are the Cartesian coordinates of the pivot point of the robot and $\theta_i$ is the robot orientation. The equations of motion for the robot $i$ are given by:
\begin{equation}
    \frac{d}{dt} 
    \begin{bmatrix}
    x_i\\
    y_i\\
    \theta_i
    \end{bmatrix}
  = a_i \cdot 
  \begin{bmatrix}
    \cos\theta_i\\
    \sin\theta_i\\
    0
    \end{bmatrix}\cdot u
    +(1-a_i)\cdot 
    \begin{bmatrix}
    0\\
    0\\
    1
    \end{bmatrix}\cdot \frac{u}{r_i}
    \label{eq:continuous}
\end{equation}
where $a_i \in \{0,1\}$ is the switching input that determines whether the robot $i$ is moving forward or turning in place, $u \in \mathbb{R}^+$ is the rate of forward motion, and $r_i$ is the turning radius of the robot $i$. If the control input is piecewise constant over each epoch $\Delta T$, we obtain the following discrete-time model:
\begin{equation}
    % q_i(k+1) =
    \begin{bmatrix}
    x_i(k+1)\\
    y_i(k+1)\\
    \theta_i(k+1)
    \end{bmatrix}
    =
    \begin{bmatrix}
    x_i(k)\\
    y_i(k)\\
    \theta_i(k)
    \end{bmatrix}+
    \begin{bmatrix}
    a_{i}(k)  \cos\theta_i(k)\\
    a_{i}(k)  \sin\theta_i(k)\\
    (1-a_{i}(k))  1/r_i
    \end{bmatrix} u(k) \cdot \Delta T.
    \label{eq:discrete}
\end{equation}
Note that $u$ is a unilateral control input. In other words, MicroStressBot cannot go backward or turn clockwise.
    
\subsection{Switched and Embedded Systems}

{Group-control} corresponds to setting the switching input $a_i$ of the robot $i$ to $1$ if it belongs to the activated group or to $0$ otherwise. {Group-control} with $m$ groups thus turns the swarm into a switched system with $m$ discrete states. Such \emph{$m$-switched system} has a system state $q(t) \in \mathbb{R}^N$ that evolves according to the following dynamics:
\begin{equation}\label{eq:pre_switched}
    \dot{q}(t) = f_{v(t)}(t, q(t), u(t)), \quad q(t_0) = q_0,
\end{equation}
where at each $t>t_0$, the switching control $v(t) \in  \{ 1,2,\ldots, m \}$, and the control input $u(t) \in \mathbb{R}^M$ for some $M$ (in our case $M=1$). In the case of $n$ robots that move in $SE(2)$, we have $N=3 n$. The vector fields $ f_i: \mathbb{R} \times \mathbb{R}^N \times \mathbb{R}^M \rightarrow \mathbb{R}^N, \quad i \in \{ 1,2,\ldots, m \}$ are $\mathcal{C}^1$ and we assume that $u(t)$ is measurable. Note that the evolution of the state $q(t)$ governed by Eq.~\eqref{eq:pre_switched} does not experience any discontinuous jumps.

We introduce new variables $v_i(t) \in [0, 1]$ that satisfy $\sum_{i=1}^{m} v_i(t) = 1$. This transforms the switched system (\ref{eq:pre_switched}) to the embedded form~\cite{bengea2005optimal,wei2007applications}:
\begin{equation}\label{eq:pre_embedded}
    \dot{q}(t) = \sum_{i=1}^{m} v_i(t)f_{i}(t, q(t), u_i(t)), \quad q(t_0) = q_0,
\end{equation}
where $u_i$ is the control input for each vector field $f_i$. It can be shown~\cite{bengea2005optimal} that the set of trajectories of the switched system \eqref{eq:pre_switched} is a dense subset of the set of trajectories of the embedded system~\eqref{eq:pre_embedded}. Consequently, any trajectory of the embedded system can be approximated by the switched system to any desired accuracy, which means that the controllability of the switched system is equivalent to the controllability of the embedded system.

%% file: 3_Group_based_control.tex
\section{Controllability Analysis}

For the robots to navigate among obstacles, the system, in general, needs to be Small-Time Locally Controllable (STLC). Otherwise, there are configurations in the obstacle-free configuration space $C_{free}$ that will always lead to a collision. First, we formally define the notion of STLC that will be of interest in this paper. Let $q(t)$ be the combined state of the robot swarm, where all the robot positions are first stacked together in a $2 n$ vector, followed by a $n$ vector of the robot orientations. Let $p(t) = q(t)[1:2n]$ be the position states of the robots, and let $\Theta(t)=q(t)[2n+1:3n]$ be the orientations. Let $\mathcal{R}_{(p_0,\Theta_0)}(T)$ be the set of all positions reachable by the trajectories of the system starting at $q(0)=(p_0,\Theta_0)$ in time no greater than $T$.

\begin{definition}
The system is small-time locally controllable (STLC) in position states about an initial state $(p_0,\Theta_0)$, if $p_0$ is contained in the interior of the reachable sets $\mathcal{R}_{(p_0,\Theta_0)}(T)$ for all $T>0$. \cite{kawski2002combinatorics}
\end{definition}

Note that each MicroStressBot in the swarm has $SE(2)$ symmetry (if the initial configuration of the robot is moved, its trajectories will move by the same amount). So, STLC at one initial configuration implies STLC everywhere.

In Eq.~\eqref{eq:continuous}, we assumed that MicroStressBot $i$ has the turning radius ${r_i}$. If each robot has a different turning radius, only two groups are needed to achieve STLC through group-based control: group $G_1$ that translates all robots and group $G_2$ that rotates all robots. To see this, consider the three-step primitive $P_1 = (G_1(d), G_2(\pi^{r_1}), G_1(d))$ where $G_1(d)$ moves all robots forward for $d$, $G_2(\pi^{r_1})$ rotates all robots so that robot 1 rotates exactly for $\pi$ radians, and then $G_1(d)$ again moves all the robots forward for $d$, making robot 1 return to its starting position. Then the sequence $P_{2} = (P_1, G_2((\pi - \Delta\theta_2)^{r_2}), P_1)$ moves robots 1 and 2 back to their initial position, where $\Delta\theta_2$ is the angular displacement of robot 2 due to $P_1$. Iteratively, $P_{n-1}$ moves only one robot while all others return to the starting position. By first rotating the robot so that the resulting movement due to $P_{n-1}$ is in the desired direction, we can individually move any robot in any direction, showing that the system is STLC. Note that a similar argument has been used in~\cite{paprotny_turning-rate_2013} to introduce TSC.

Assuming that all the robots have different turning rates again shifts the burden of control to fabrication and generally scales poorly. In the rest of the paper, we thus \textbf{assume that all robots have the same turning rate $r$}. We will show that {group-control} still guarantees STLC.

\begin{remark}
 Given a collection of sets of robots $\{S_0, S_1,\ldots \}$ with each $S_i$ containing robots with the same turning rate, and $S_i$ and $S_j$ having different rates if $i \ne j$, STLC can be achieved by combining the approach above with the methodology that will be described below within each set.
\end{remark}

\subsection{Unilateral Control to Bilateral Control}

As stated above, in the rest of the paper, we only consider the case where all the robots have the same turning radius $r$. In~\cite{li2022group}, we showed that a swarm of MicroStressBots under {group-control} is (globally) controllable (even without using the group where all the robots rotate). However, we were unable to demonstrate that the system is STLC. In this paper, we show that by adding the group where all the robots rotate, STLC can be achieved.

One of the challenges in showing that our system is STLC is that the robots are controlled unilaterally. In other words, they can only move forward but not backward. Similarly, they can rotate counter-clockwise but not clockwise. This is the result of the control input $u$ in Eq. \eqref{eq:pre_embedded} being positive. We show that this restriction can be relaxed.

We return to the switched system in Eq. \eqref{eq:continuous}. Let \[q(t) = [x_1, y_1, x_2, y_2, \ldots, x_n,y_n, \theta_1,\theta_2 \ldots, \theta_n]\] be the state of the swarm, where $q(t) \in \mathbb{R} ^{3n}$. We can see that $q(t)[1:2 n]$ are the \emph{position states}, with $q(t)[2j-1:2j]$ representing the position of robot $j$. Also, $q(t)[2n+1:3n]$ are the \emph{orientation states}, where $x(t)[2n+j]$ is the orientation of the robot $j$. Next, let $\alpha_i=[ \alpha_{i, 1}, \ldots, \alpha_{i, n} ]^T, i=1,\ldots,m$ be the activation vector corresponding to the group $i$ being activated. In other words, $\alpha_{i, j}=1$ if robot $j$ belongs to group $i$, and $0$ otherwise. The overall swarm dynamics can then be written as follows:
\begin{equation}\label{eqn:SS}
    \dot q(t) = f_{\nu(t)}(q(t)) \cdot u(t) \quad u(t)\in \mathbb{R}^{+}
\end{equation}
where $\nu(t) \in \{1,\ldots,m\}$, and for each $i$, $f_i$ is obtained by choosing $a_j=\alpha_{i, j}$ in Eq. \eqref{eq:continuous}. This equation describes a switched driftless control-affine system~\cite{liberzon2003switching, bullo_geometric_04}.

We introduce the notion of induced vector fields $g_i$ and $h_i$ such that $f_i = g_i + h_i$. Here, $g_i$ contains only position state information ($f_i[2n+1:3n]=0$), and $h_i$ contains only the orientation state information ($h_i[1:2n]=0$). Given that $G_m$ is the group where all the robots rotate and their rotation radius is the same, let $f_m(\pi)$ represent the rotation of all the robots by $\pi$ radians using the rotation-only vector field. 

\begin{proposition}
The control sequence $(f_i(d), f_m(\pi),f_i(d), f_m(\pi))$, where $f_i(d)$ corresponds to the activation of group $G_i$ so the robots in the group translate for $d$, corresponds to a vector field $h_i(\frac{2 d}{r})$ that rotates the robots that do not belong to the group $i$ by $\frac{2 d}{r}$ and keeps the rest of the robots where they were.
\end{proposition}

\begin{proof}
  Observe that $f_i(d)$ translates the robots that are in $G_i$ for $d$ without changing their orientation and rotates the robots that are not in $G_i$ for $\frac{d}{r}$. The application of $f_m(\pi)$ rotates all the robots for $\pi$, which means that the robots that are in $G_i$ now point in the opposite direction. The second application of $f_i(d)$ then returns the robots that are in $G_i$ to their initial position while adding another $\frac{d}{r}$ to the orientation of the robots that are not in $G_i$. Finally, $f_m(\pi)$ rotate the robots in $G_i$ to their original orientation. The robots not in $G_i$ have instead rotated in total for $\frac{2 d}{r}+2 \pi=\frac{2 d}{r}$. \qed
\end{proof}

\begin{remark}
    Observe that for any angle $\theta \in [0, 2 \pi]$, $h_i(-\theta)= h_i(2 \pi - \theta)$. The orientation vector field $h_i$ is thus bilateral.
\end{remark}

Once we have the orientation vector field $h_i$, we can easily obtain the translation vector field $g_i$:
\begin{proposition}
The control sequence $(f_i(d),h_i(-\frac{d}{r}))$ corresponds to a vector field $g_i(d)$ that translates the robots in $G_i$ for $d$ and leaves all the other robots where they were.
\end{proposition}

\begin{remark}
    We can also easily see that $(f_m(\pi),g_i(d),h_i(\pi))=g_i(-d)$. Thus, the translation vector field $g_i$ is bilateral.
\end{remark}

Note that while $h_i$ rotates the robots that are not in $G_i$ in place, the other robots travel back and forth and could potentially collide with an obstacle. However, we can restrict the motion of the robots in $G_i$ to a small ball of radius $\epsilon$ and apply $h_i(\frac{\epsilon}{r})$ multiple times to obtain the desired rotation of the robots not in $G_i$ without the robots in $G_i$ leaving a small neighborhood of their initial location.

Technically, the rotations we used in the constructions above cannot be performed in zero time. However, for the purpose of position control, they can be assumed to take place instantaneously. Therefore, in the rest of the paper, we can safely assume that the control input can be bilateral. 

With the construction above, the original embedded system in Eq. \eqref{eq:pre_embedded} with unilateral vector fields $\{f_i,\ldots,f_m\}$ becomes an embedded system with \textbf{bilateral} vector fields
\begin{equation}
\mathcal{F} = \{g_1, \ldots, g_{m-1}, h_1, \ldots, h_{m-1}, f_m\}.
\label{Eq:VecFields}
\end{equation}

\subsection{Control of Robot Orientation}\label{Sect:Rotation}

We have shown that each group $G_i$ induces a vector field $g_i$ that only translates the robots in the group and a vector field $h_i$ that just rotates the robots that are not in $G_i$. What we show next is that we can independently control the orientation of at least 3 robots:

\begin{proposition}\label{prop:angle_dof}
    The orientation of any 3 robots can be controlled to any desired orientation $(\theta_1,\theta_2,\theta_3)$. Moreover, at most $m$ robots can be controlled to any desired orientation $(\theta_1,\ldots,\theta_m)$.
\end{proposition}

\begin{remark}
    The robots that are not directly controlled will rotate in place to some unspecified orientation.
\end{remark}

\begin{proof}
Consider the rotation vector fields $h_1, \ldots, h_{m-1}$ in \eqref{Eq:VecFields} and the additional original rotation-only vector field $f_m$. All of these vector fields can rotate the robots without translating them. The proposition can then be proved by observing that, using these vector fields, the orientations of the robots linearly depend on the arguments (rotation angles) of these vector fields. For example, for the 4-group case in Table~\ref{Tab:Groups}, we get the equation:
\begin{equation}
    \begin{bmatrix}
% 1 & 1 & 1 & 0 & 0 & 0\\
% 1 & 0 & 0 & 1 & 1 & 0\\
% 0 & 1 & 0 & 1 & 0 & 1\\
% 1 & 1 & 1 & 1 & 1 & 1

1 & 1 & 0 & 1 \\
1 & 0 & 1 & 1 \\
1 & 0 & 0 & 1 \\
0 & 0 & 1 & 1 \\
0 & 1 & 0 & 1 \\
0 & 1 & 1 & 1 
\end{bmatrix} \cdot \begin{bmatrix}
    u_1 &u_2& u_3&u_4
\end{bmatrix}^T
 = 
 q[2n+1:3n]-q[2n+1:3n](0)
\label{eqn:rotation}
\end{equation}
where $q[2n+1:3n]$ are the orientations of the robot after the application of the vector fields, and $q[2n+1:3n](0)$ are their initial orientations.

It is evident that for $m$ groups, the rank of the matrix in the equation above is at most $m$. This implies that at most $m$ of the $n$ equations can be solved exactly for the desired inputs $u_i$, indicating that the orientation can be controlled for at most $m$ robots.

Next, we illustrate that the orientation of any 3 robots can be controlled. To see this, note that if one considers the first $m-1$ columns of the matrix, any two rows are distinct (since no two robots belong to the same set of groups), which means that at least two of them are linearly independent. Adding the last column with all ones, it is clear that any additional row will be linearly independent from these two rows, which means that the matrix will have a rank of at least 3 and that any 3 of the $n$ equations can be solved exactly for the desired inputs $u_i$. \qed
\end{proof}

%% file: 4_STLC_proof.tex
\subsection{Small-Time Local Controllability (STLC)}\label{sec: control_proof}

After we have identified the control vector fields in Eq.~\eqref{Eq:VecFields}, we can show that the swarm is STLC in positions. Traditionally, we can do this by analyzing the closure of $\mathcal{F}$ under the Lie bracket operation. However, we will show that we can independently translate any robot in an arbitrary direction without moving any other robot. This shows that the system is STLC in positions.

Suppose we want to independently translate some robot $k \in \{1, \ldots,n\}$. Let $G_{\hat{k}}$ be any group that contains the robot $k$. Let $\operatorname{Rot}_{i_1,i_2,i_{3}}(\theta_1,\theta_2,\theta_{3})$ correspond to setting the orientation of robots $i_1,i_2,i_{3}$ to $\theta_1,\theta_2,\theta_{3}$ as described above. Also, assume that the members of the group $G_{\hat{k}}$ are robots $l_1,\ldots,l_j,k$.

We can then state the following:
\begin{proposition}\label{Thm:STLC}
    The control sequence $P_1(d)=(g_{\hat{k}}(d),\operatorname{Rot}_{l_1,l_{2},k}(\pi,\pi,0),g_{\hat{k}}(d))$ only translates robots $l_{3},\ldots,l_j,k$.
\end{proposition}

\begin{proof}
Under the control sequence above, the first $2$ robots in $G_{\hat{}}$ will travel back and forth while the robot $k$ will undergo translation for $2  d$. The rest of the robots in the group will undergo some unspecified motion.

Once we have the control sequence $P_1$, it is easy to see that a control sequence $P_2 = (P_1(d),\operatorname{Rot}_{l_{3},l_{4}, k}(\pi-\Delta\theta_{l_{3}},\pi-\Delta\theta_{l_{4}}, 0),P_1(d))$ will only translate robots $l_{5},\ldots,l_j,k$, where $\Delta \theta_{l_{3}}, \Delta \theta_{l_{4}}$ are rotations of robots $l_3, l_4$ caused by $P_1(d)$. Repeating this process at most $l_j/2$ times allows us to eliminate the translation of all robots $l_1,\ldots,l_j$ except for the robot $k$. Of course, we can always initially rotate the robot $k$ so that the final translation is in the desired direction. This shows that each robot can be translated independently and that the system is STLC in positions.\qed
\end{proof}

%% file: 5_Motion_planning.tex
\section{Motion Planning}

      % \vspace{-10mm}

% A common approach of robotics planning in an obstacle environment is the sampling-based method. For group-based microrobots using RRT, the typical method involves selecting and moving one group within the RRT local planner, and then adding nodes to the tree accordingly. While this method is viable, it has significant planning complexity. Therefore, we investigate how system constraints can inform the effective use of the sampling-based RRT method for group-based microrobots.

After we have shown that the MicroStressBot swarm under {group-control} is STLC, we consider the motion planning problem for the swarm. Given our proof of STLC, the problem is not how to move the swarm to a particular configuration, but how to move it \textbf{efficiently}. In other words, ideally, we would like the robots to move to their desired configuration (only the position matters) in parallel and not one by one. However, planning such parallel motion directly is complex due to the high dimension of the configuration space and the fact that we have a single control input. In this paper, we argue that effective motion planning for global {group-control} in high-dimensional state spaces should involve the design of \textbf{motion primitives}. In particular, motion primitives will be associated with \textbf{subgroups}, trading off the complexity of implementing additional groups in hardware (using PFSMs) with the complexity of control.

The idea of motion primitives is to divide a robot group into smaller subgroups. These primitives simplify the motion planning task by reducing the coupling between robots in the group. We present primitives inspired by Lie bracket motions, following the intuition that, in general, the higher the order of the Lie bracket, the fewer robots are moved by it. It is important to note that, in this paper, these subgroups are logical constructs as opposed to $\log (n+2)  +1$ physical groups. However, these logical groups could be implemented in hardware, increasing the number of physical groups and simplifying motion planning and control for the swarm. Of course, this would come with increased fabrication complexity and cost.

%The order of a Lie bracket denoted $r_k$, serves as a measure of complexity, where $r_k \in \{ 1,\dots, r_s\}$ and $r_s$ is the lowest Lie bracket order that ensures STLC. As the motion planning process incorporates higher-order Lie brackets, the corresponding controls become more complex. Nevertheless, increasing the number of decoupled subgroups simplifies the planning task by reducing the need for coupled motions.

In the following section, we are interested in the complexity of the motion planning problem in terms of computation time, path length, and execution time. We investigate how the motion planning problem can be simplified by using motion primitives through both a theoretical complexity bound and simulated instantiations. A pivotal aspect is that a primitive corresponding to a smaller subgroup will generally involve the recursive application of primitives associated with larger subgroups, as observed in the proof of Proposition~\ref{Thm:STLC}. Because of this, such primitives will involve many back-and-forth motions of the robots, resulting in longer path lengths and execution times. However, motion plans that involve such primitives will be simpler, which means a shorter computation time. On the other hand, when the primitives used for motion planning involve larger subgroups (not much back-and-forth), the motion of the robots is coupled, which increases the computation time for motion planning but minimizes back-and-forth motions and reduces path length and execution time.

\subsection{Motion Planning Approximation Scheme}

% We start by giving some assumptions and definitions that are fundamental to the {group-control} framework. We start with the following \textbf{assumptions}:
% \begin{enumerate}
%     \item Robots are controlled by a global signal that operates them in groups, with one group activated at a time.
%     \item A predefined map of the area is available and known.
%     \item Each robot is assigned a starting point and a goal configuration. The starting point includes the robot's position and orientation, whereas the goal configuration only involves position; the orientation can be arbitrary.
% \end{enumerate}

Unlike physical groups (encoded through PFSMs), a subgroup and the associated motion primitive are a logical construct. A subgroup contains fewer robots than physical groups, and only those robots move forward when a subgroup is activated. This abstraction will allow us to formulate the motion planning problem in layers of increasing difficulty. Physically, the motion induced by a given subgroup $p$ corresponds to a particular Lie bracket and can thus be physically realized. We will refer to such motion as a \emph{motion primitive associated with the subgroup $p$}.

%refer to When using a motion primitive associated with a subgroup $p$ we will use the notation $<p,l>$ to indicate that the robots in the subgroup $p$ moved forward for the distance $l$.

% , which consists of a control vector field or a Lie bracket direction $f$, and a length $l$. This pair is realized through a sequence of control actions represented as $<p, l> = {<a_1,u_1>, <a_2,u_2>, ...}$, where $a_i$ denotes the group switched at time $i$ and $u_i$ denotes the linear input at time $i$.

\begin{definition}
    \textbf{[Primitive Order]} The order of a primitive corresponding to the subgroup $p$ is the order of the Lie bracket used to realize it. The order of the Lie bracket is the number of generators from the set of vector fields $\mathcal{F}$ in Eq.~\eqref{Eq:VecFields} in it.
\end{definition}

 In general, the higher the order of a Lie bracket, the fewer robots will move by that bracket. However, there is no one-to-one correspondence between the order of the Lie bracket and the number of robots it moves, as this also depends on the vector fields involved in the Lie bracket. Fundamentally, the order of the Lie bracket indicates the difficulty of its implementation; higher order is more difficult to implement. On the other hand, the number of robots affected by a Lie bracket reflects the complexity of motion planning; the more robots move simultaneously, the more difficult it is to move each of them to the desired configuration.
    % See section \ref{sec: control_proof} for the definition of order of Lie bracket. 
    
% For formal definition of the order of Lie bracket see \cite{Kawski,Sussman}. 
% \mz{You need to formally explain what's the relationship btw. the order of the Lie bracket and the number of robots that move under that bracket.} \SL{done}

% \begin{proposition}

% \end{proposition}

\begin{definition}
\textbf{Motion Planning Problem Abstraction} The motion planning problem for $n$ robots of level $r$, denoted as $\mathcal{M}_{q_s}^{q_g}(n, r)$, corresponds to the problem of finding a collision-free trajectory (sequence of motion primitives) for $n$ robots controlled by a global field from some initial state $q_s$ to a goal state $q_g$, where the motion primitives that are used have the primitive order at most $r$. 
\end{definition}

% Further, we provide some theoretical insights into motion planning complexity for the microrobot System.

\begin{proposition}
     Consider the motion planning problem $\mathcal{M}_{q_s}^{q_g}(n,k)$ for $n$ robots and $1 \leq k\leq k_{max}$, where $k_{max}$ is the primitive order required for STLC. If $k_1 < k_2 \leq k_{max}$, the complexity of $\mathcal{M}_s^g(n,k_2)$ is less than $\mathcal{M}_{q_s}^{q_g}(n,k_1)$. When $k = k_{max}$, the complexity of $\mathcal{M}_{q_s}^{q_g}(n,k_{max})$ solved by RRT~\cite{lavalle_randomized_2001} is $O(n\cdot c^n)$. For any other $k$, the complexity of $\mathcal{M}_{q_s}^{q_g}(n,k)$ solved by RRT is $O(n\cdot c^n\cdot L^{k_{max}-k})$, where $L$ is the complexity of the first-order Lie bracket and $c$ is a constant depending on the environment size and resolution.
\end{proposition}

\begin{remark}
    By observing that the primitives used in the proof of Theorem~\ref{Thm:STLC} are analogous to a Lie bracket operation, it can be shown that $k_{max}<=m$. The formal proof relies on the special properties of the vector fields in $\mathcal{F}$ (e.g., $[h_i,h_j]=[g_i,g_j]=[h_i,g_i]=0$ for any $i,j<m$) and is beyond the scope of this paper.
\end{remark}

%\begin{remark}
%It can be shown that $r_s=...$.
%\end{remark}

\begin{proof}
For \( k = k_{max} \), the system is fully actuated due to STLC, allowing motion in all directions of the configuration space. For position control of $n$ robots, the dimension of the configuration space is \( d = 2n \). As the complexity of representing the space is exponential in $d$, the number of samples \( n_s \) required for RRT scales as \( O(c^n) \), resulting in the complexity \( O(n_s \log n_s) = O(c^n \log(c^n)) = O(n\cdot c^n) \) for the RRT solver.

When \( k < k_{max} \), the system is underactuated, meaning that motion in certain directions cannot be directly achieved. Instead, these motions are approximated through paths composed of Lie brackets of control vector fields. The Lie bracket is implemented by alternating the motion back and forth along two vector fields. If $L$ is the complexity of this maneuver for the first-order Lie bracket, higher-order Lie brackets correspond to nested applications of an analogous maneuver. Given a Lie bracket of order $k$, the motion in a given direction must ultimately be implemented through the Lie brackets of the order $k_{max}$. The complexity of this implementation is $L^{k_{max}-k}$.  Therefore, the overall complexity of the RRT algorithm becomes \( O(n_s  L^{k_{max}-k} \log n_s) = O(n \cdot c^n \cdot L^{k_{max}-k} ) \).
% Thus, the complexity to get $n_s = O(c^n)$ effective number of samples is \( O(c^n L^{k_{max}-k} ) \). The total complexity of RRT algorithm becomes \( O(n_s \log n_s) = O(c^n L^{k_{max}-k} \log(c^n L^{k_{max}-k})) \).

As \( k \) approaches \( k_{max} \), the complexity reduces since the system becomes more actuated, requiring fewer nested Lie bracket computations. Consequently, the complexity of \( \mathcal{M}_{q_s}^{q_g}(n,k_2) \) is less than \( \mathcal{M}_{q_s}^{q_g}(n,k_1) \) for \( k_1 < k_2 \leq k_{max} \). \qed
\end{proof}

\begin{example}
    Continuing with the example of 6 robots and 4 groups, consider the motion planning problem $\mathcal{M}(6,2)$. In this case, the minimum order to move each robot is $r_{s}=2$ (see Table~\ref{Tab:bracketeachrobot}). We can select six primitives involving these second-order Lie brackets that move each robot individually.
    %, as shown in Table~\ref{Tab:bracketeachrobot}.
    The problem $\mathcal{M}(6,2)$ then reduces to rotating each robot toward the goal and moving it using its corresponding primitive.
    
    For the motion planning problem $\mathcal{M}(6,1)$, the original vector fields from $\mathcal{F}$ are chosen as the primitives. Each primitive moves three robots simultaneously. For example, $g_1$ moves robots 4, 5, and 6, and $g_3$ moves robots 1, 3, and 5.
    %Moving each robot individually is computationally much more straightforward than finding a path that moves three robots at any given time. 
    To move each robot to its desired configuration using these primitives requires the application of the Lie bracket $[h_1+h_2-f_4, g_3]$ utilizing a sequence of primitives $g_3$ and $h_1+h_2 f_4$.
    %Thus, moving one robot requires a complexity of $\Omega(L^{2-1}) = \Omega(L)$ and the overall complexity of $\mathcal{M}(6, 2)$ is $\Omega(nL)$. 
\end{example}

\begin{table}[htb]
\vspace{-6mm}
\begin{center}
 \begin{tabular}{|| c | c ||} 
  \hline
  Lie brackets &   Robots\\ [0.5ex] 
 \hline\hline
   $[h_2, g_1]$  &  4, 5\\
 \hline
$[h_3, g_1]$ & 4, 6 \\
 \hline
  $[h_1, g_2]$ & 2, 3\\
 \hline
    $[h_3, g_2]$ & 2, 6 \\
 \hline
    $[h_1, g_3]$ & 1, 3 \\
 \hline
 $[h_2, g_3]$ & 1, 5 \\
 \hline
\end{tabular} \hspace{1cm} \begin{tabular}{|| c | c ||} 
  \hline
  Lie brackets &   Robots \\ [0.5ex] 
 \hline\hline
 $[h_1+h_2-f_4, g_3]$  & 1\\
 \hline
$[h_1+h_3-f_4, g_2]$ & 2\\
 \hline
  $[f_4-h_3,g_2]$ & 3\\
 \hline
    $[h_2+h_3-f_4,g_1]$ & 4 \\
 \hline
    $[f_4-h_3,g_1]$ & 5 \\
 \hline
 $[f_4-h_2,g_1]$ & 6\\
 \hline
 \end{tabular}
\end{center}
%\vspace{-2mm}
\caption{Lie brackets for $m=4$ and robots that are moved by them.}
\vspace{-9mm}
\label{Tab:bracketeachrobot}
\end{table}

% \begin{table}[htb]
% %\vspace{-2mm}
% \begin{center}
%  \begin{tabular}{|| c | c ||} 
%   \hline
%   Lie brackets &   Robot\\ [0.5ex] 
%  \hline\hline
%  $[h_1, [h_2, g_3]]$  & $-c_1,-s_1$\\
%  \hline
% $[h_1, [h_3, g_2]]$ & $-c_2,-s_2$\\
%  \hline
%   $[h_1, [h_3, g_2]]- [h_2,[h_2, g_1]]$ & $-c_3,-s_3$\\
%  \hline
%     $[h_3, [h_2, g_1]]$ & $-c_4,-s_4$ \\
%  \hline
%     $[h_1, [h_2, g_3]] - [h_1, [h_1, g_3]]$ & $-c_5,-s_5$ \\
%  \hline
%  $[h_3, [h_2, g_1]] - [ h_2, [h_2, g_1]]$ & $-c_6, -s_6$\\
%   \hline\hline
%    $[h_2, g_1]$  & $(-s_4,c_4),(-s_5,c_5)$\\
%  \hline
% $[h_3, g_1]$ & $(-s_4,c_4),(-s_6,c_6)$\\
%  \hline
%   $[h_1, g_2]$ & $(-s_2,c_2),(-s_3,c_3)$\\
%  \hline
%     $[h_3, g_2]$ & $(-s_2,c_2),(-s_6,c_6)$ \\
%  \hline
%     $[h_1, g_3]$ & $(-s_1,c_1),(-s_3,c_3)$ \\
%  \hline
%  $[h_2, g_3]$ & $(-s_1,c_1),(-s_5,c_5)$\\
%  \hline
% \end{tabular}
% \end{center}
% %\vspace{-2mm}
% \caption{In the case of 6 robots in 4 groups, the first six rows indicate Lie brackets that move each individual robot ($r_k = 3$), while the last six rows show the Lie brackets of order $r_k = 2$. In the right column, only then non-zero elements of each Lie bracket vector field are shown; $c_i$ and $s_i$ correspond to $\cos{\theta_i}$ and $\sin{\theta_i}$ respectively, where $\theta_i$ is the orientation of robot $i$.}
% \vspace{-8mm}
% \label{Tab:bracketeachrobot}
% \end{table}

In short, motion planning is simplified if fewer robots are moved by a primitive; however, this approach compromises path efficiency. The computational complexity of motion planning increases exponentially as primitive order decreases and the motion of more robots is coupled; however, the robots move in parallel, so the paths are more efficient.

\section{Instantiations of the motion planning  problem}

The motion planning problem $\mathcal{M}(n, k)$ involves moving $n$ robots with primitives of order $k$. This section investigates some instantiations of the motion planning problem $\mathcal{M}(n, k)$. 

For a given set of $n$ robots, we use $m$ groups to control them. Depending on the chosen group allocation, the set of all primitives up to the order $k$ is defined as $\mathcal{P}^k = \bigcup_{i\le m} \{f_i\} \cup \bigcup_{2 \le i\le k} S^i(\mathcal{F}) \}$, where $S^i(\mathcal{F})$ are subgroups that correspond to Lie brackets of the order $i$. By choosing different subsets of $\mathcal{P}^{k}$, we will obtain different instantiations of the motion planning problem. 
% A motion planning problem instance focuses on higher-level planning based on the given primitives, rather than the lower-level execution of those primitives.
We next discuss possible implementations of primitives and then investigate several instances of the motion planning problem.

\subsection{Implementation of Primitives}
\label{Sect:Implementation}

Next, we compare two approaches for implementing the primitives: (a) numerical optimization, and (b) primitives designed by hand. In numerical optimization, we choose a large enough number of control steps and randomly decide which group is active during each control step. We then use numerical optimization to compute the time for which each group is activated (the distance for which robots in the group move) so that the total path length is minimized. This approach is similar to the approach described in~\cite{li2022group}. In contrast, hand-designed primitives follow the procedure described in the proof of Proposition~\ref{Thm:STLC}, where robots are eliminated one by one from the original group by using back-and-forth motion.

% \begin{figure}[t]
%     \centering
%     \includegraphics[width=0.7\linewidth]{figs/primitives.pdf}
%     \vspace{-2mm}
%     \caption{Motion planning problem $M(6,2)$ with the primitive corresponding to $[h_2,g_1]$. Top: numerical optimization; Bottom: hand-designed primitive.}
%     \label{Exper:primitives}
%     \vspace{-6mm}
% \end{figure}

 Figure~\ref{Exper:primitives} shows the results in the case of the motion planning problem $M(6,2)$, where we have 6 robots and 4 groups, and the particular motion primitive that we are interested in is the Lie bracket $[h_2, g_1]$\footnote{All the figures indicate the starting location of each robot with a green circle and the goal location with a red cross.}. This primitive moves robots 4 and 5 (Table~\ref{Tab:bracketeachrobot}). Therefore, the subgroup has 2 robots ($\{4,5\}$), compared to the original group $G_1$ that contains 3 robots (4, 5, and 6). The results of numerical optimization are shown on the left. The number of control steps was chosen to be $k=35$. A group is randomly selected for each control step, and then we use numerical optimization to compute the activation time for each group (length of motion) so that the total path length is minimized and that the robots move for some small distance $\epsilon$ in the direction of the target configuration. The distance to the target configuration was $1$, and $\epsilon$ was chosen to be $0.2$; this means that the computed primitive was repeated 5 times to reach the desired goal configuration. The right panel shows the hand-designed primitive using
% In contrast, in this paper, hand-designed primitives, involve an iterative application of a 3-step control sequence: in this paper. The sequence can be generalized as follows:
% \begin{align*}
%     & P^i_1(d) = f_i(d), \\
%     & P^i_2(2d) = (P^i_1(d), \operatorname{Rot}_{21,22,23}(\Delta\theta_{21}, \Delta\theta_{22}, \Delta\theta_{23}), P^i_1(d)), \\
%     & \ldots \\
%     & P^i_k(2^{k-1}d) = (P^i_{k-1}(2^{k-2}d), \operatorname{Rot}_{k1,k2,k3}(\Delta\theta_{k1}, \Delta\theta_{k2}, \Delta\theta_{k3}), P^i_{k-1}(2^{k-1}d)),
% \end{align*}
% where $f_i$ represents the control vector field for group $i$, and $\operatorname{Rot}$ performs a rotation of three robots at a time. As demonstrated in the proof of STLC, by carefully selecting the rotation angles, some robots can be moved forward while others are returned to their original positions. Each nested primitive is capable of moving at most two robots (given that in order to move one robot, we maintain its direction in every layer), and for systems involving more than two robots, a combination of nested primitives can be applied in sequence. 
%Continuing the earlier example with robots 4 and 5 in a subgroup,
the control vector field $f_1 = \left[\mathbf{0}_6, c_4, s_4, c_5, s_5, c_6, s_6, \mathbf{0}_3, \frac{1}{r}, \frac{1}{r}, \frac{1}{r}\right]^T$ (where $c_i\equiv \cos \theta_i$ and $s_i\equiv \sin \theta_i$). The primitive is defined as $P^1_2 = (f_1, \operatorname{Rot}_{r_4, r_5, r_6}(0, 0, \pi), f_1)$. The operation $\operatorname{Rot}_{r_4, r_5, r_6}(0, 0, \pi)$ rotates robot 6 by $\pi$ while keeping robots 4 and 5 in place. This rotation involves the rotational vector fields $h_i$ as described in Section~\ref{Sect:Rotation}. 
All robots start from the origin and are oriented along the positive $x$-axis. It can be seen that the hand-designed nested primitives in the right panel yield a much simpler path. Moreover, solving optimization problems in high dimensions is challenging.

\begin{figure}[t]
\begin{center}
     \includegraphics[width=0.3 \textwidth]{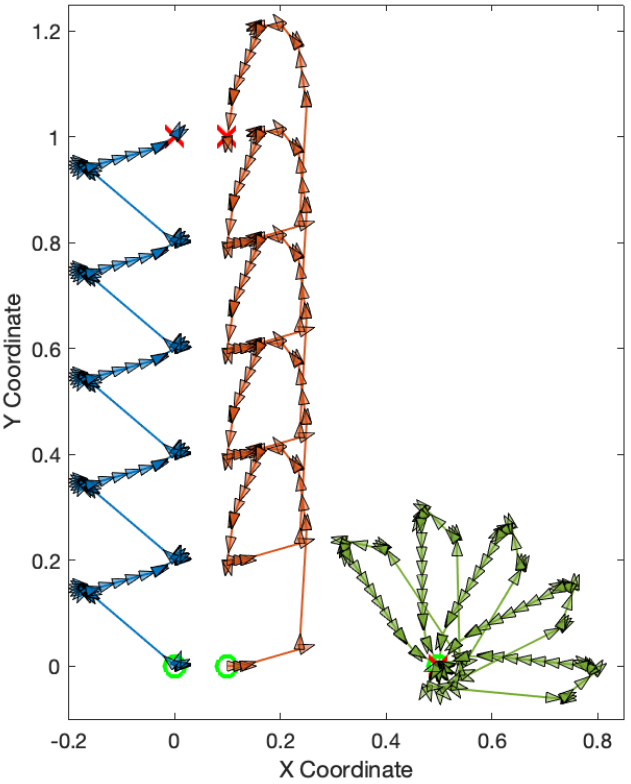} \hspace{5mm}
         \includegraphics[width=0.3 \textwidth]{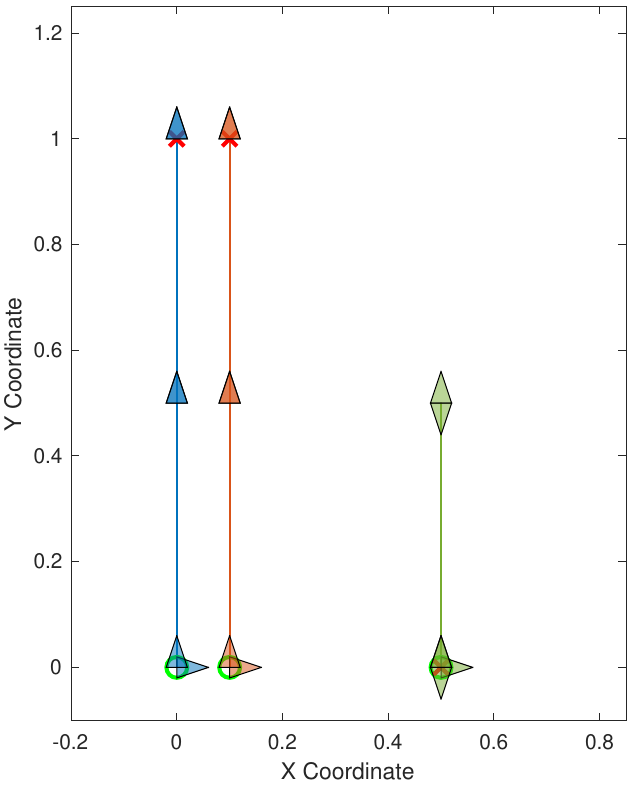}
\end{center}
\vspace{-5mm}
        \caption{Motion planning problem $M(6,2)$ with the primitive corresponding to $[h_2,g_1]$ (robots 4 and 5 move, robot 6 moves back-and-forth). Left: numerical optimization. Right: hand-designed primitive.}
        \vspace{-5mm}
        \label{Exper:primitives}
\end{figure}

\subsection{Motion planning Cases}

Next, we investigate time complexity and efficiency trade-offs for our system's motion planning problem. Our baseline motion planning algorithm is Rapidly-Exploring Random Trees (RRT)~\cite{lavalle_randomized_2001}. RRT is a widely used sampling-based algorithm that constructs a tree of feasible trajectories by incrementally extending it from an initial state toward randomly sampled points in the configuration space. The tree grows by iteratively sampling the points (nodes) and connecting them to the nearest node in the existing tree using a local planner. In our formulation, the local planner uses subgroups. In some cases, it further simplifies the problem by rotating the robots in the subgroup so they are directed toward the sampled points before using the motion primitive corresponding to the subgroup.

\begin{figure}[t]
\begin{center}
\includegraphics[width=0.32\textwidth]{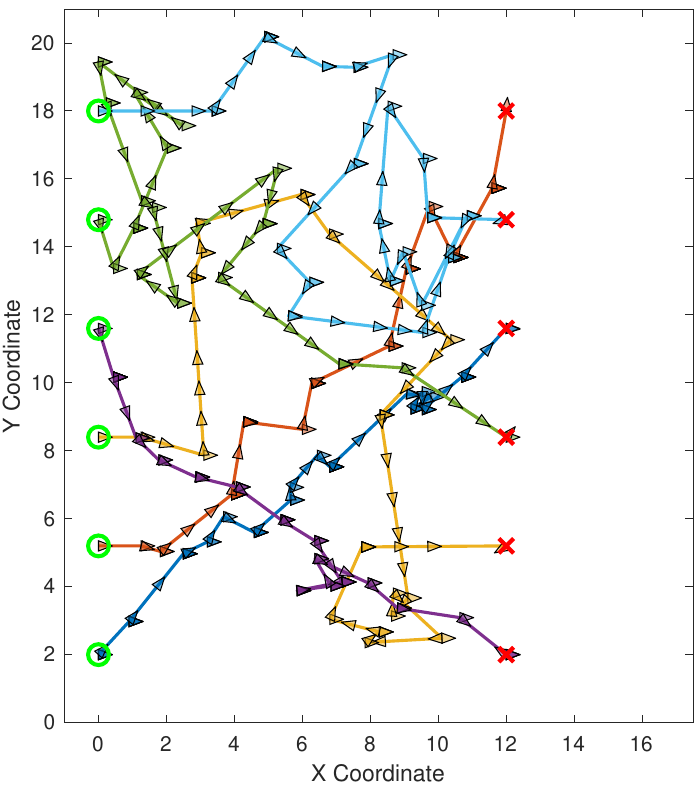} \hfill
\includegraphics[width=0.32\textwidth]{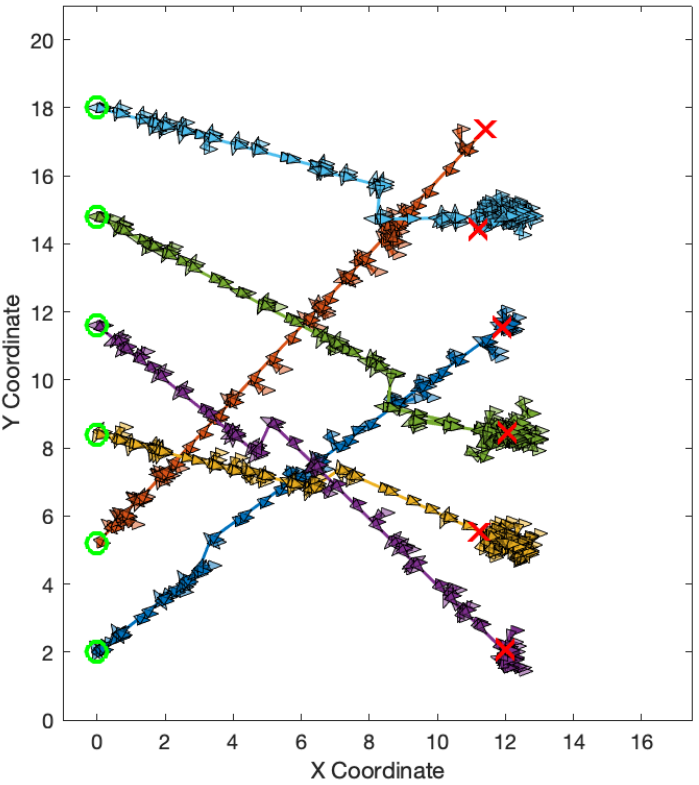} \hfill
\includegraphics[width=0.32\textwidth]{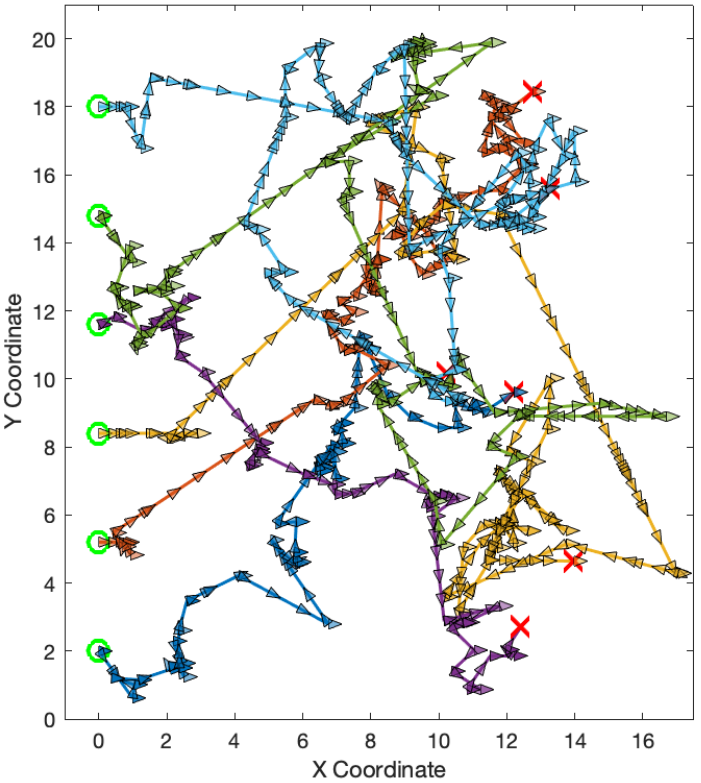}\\[-9mm]
\makebox[0.32\textwidth][r]{\scriptsize a\hspace{2mm}} \hfill
\makebox[0.32\textwidth][r]{\scriptsize b\hspace{2mm}} \hfill
\makebox[0.32\textwidth][r]{\scriptsize c\hspace{2mm}}
\end{center}
%\vspace{-5mm}
\caption{EC1: (a) Numerical optimization (exact solution); (b) RRT with rotation; (c) Original RRT.}
\label{exper:6robot_pureplan}
\vspace{-6mm}
\end{figure}

% \begin{figure}
%     \centering
%    \subfloat[]{%
%         \includegraphics[width=0.30\linewidth]{figs/fminconplan_onobs.pdf}}
%   \subfloat[]{%
%         \includegraphics[width=0.34\linewidth]{figs/Pureplan_onobs.pdf}}
%   \subfloat[]{%
%         \includegraphics[width=0.35\linewidth]{figs/Pureplan_withf.pdf}}
%     \caption{(a) Physical path for a fmincon plan (solve exactly); (b) Physical path for a RRT pure plan with controlling rotate of 3 robots; (c) Physical path for a RRT pure plan with original group swi ches. RRT terminates at a neighborhood radius of 2.}
%     \label{fig:enter-label}
% \end{figure}

\begin{table}[b]
\vspace{-4mm}
\begin{center}
    \begin{tabular}{|r|r|r|r|r|}
        \hline
        \textbf{Algorithm} & \textbf{Runtime} & \textbf{RRT Nodes}& \textbf{Path Length} & \textbf{Execution}  \\
        \hline
        Numerical Opt. & 0.80s & NA & 181.04 & 76.33s \\
        \hline
        RRT with rot. & 8.88s  & 147.20 &  323.57 & 172.36s\\
        \hline
        Original RRT$^2$ & 393.53s  & 7736.80 &  332.10 & 110.70s\\
        \hline
    \end{tabular}    
\end{center}
%    \vspace{2mm}
    \caption{EC1 Algorithm Performance (average of 20 instances). Algorithm 1 is exact, Algorithms 2 and 3 terminate within a neighborhood of radius 2. $^2$No solution was found for Original RRT in 15 out of 20 cases within the allotted time.}
    \label{Tab:RRTvsFmincon}
%    \vspace{-10mm}
\end{table}

\subsubsection{Extreme Case 1 (EC1): Pure planning}
In this scenario, there are no additional subgroups; only the original groups are used. Thus, each robot belongs to more than one group, which means the robot motions are coupled. Figure~\ref{exper:6robot_pureplan} compares numerical optimization (see Section~\ref{Sect:Implementation}) and RRT planning in the case of 4 groups and 6 robots. Numerical optimization does not consider collisions but finds the path that reaches the target exactly. Both RRT with rotation (Fig.~\ref{exper:6robot_pureplan}b) and Original RRT (Fig.~\ref{exper:6robot_pureplan}c) use the original groups as local planners and find a path that reaches the target within a certain neighborhood radius. The difference is that RRT with rotation aligns the robots with the intermediate goal point before moving them, while Original RRT does not, so it searches in higher-dimensional space. Initially, the robots are deployed in the range $y = [0,10]$ along $x = 0$, and the final goal positions are along $x = 15$, with the $y$ coordinates permuted randomly. Specifically, robots have start positions $[(0,2.0),(0,5.2),(0,8.4),(0,11.6),(0,14.8),(0,18.0)]^T$ and goal positions $[(12, 11.6),(12, 18.0),(12, 5.2),(12, 2.0),(12, 8.4),(12, 14.8)]^T$. The density of the robot icons on the path shows the efficiency of the motion. Numerical optimization provides an efficient path and icons are sparse. RRT with rotation has clusters of icons on the path; each cluster illustrates how robots move when we only control the orientations of some robots. Pure RRT has a high density of icons on the path, which means it is less efficient.
Table~\ref{Tab:RRTvsFmincon} compares the performance of these three algorithms in terms of computation, path length, and execution time. The results are reported for an average of 20 trials. Numerical optimization is the fastest and produces the shortest path. Unfortunately, it becomes computationally prohibitive for collision avoidance and in complex environments. As for the other two algorithms, RRT with rotation has a much faster runtime compared to RRT, but it incurs a penalty in the execution time.

\begin{figure}[t]
\begin{center}
        \includegraphics[width=0.32\textwidth]{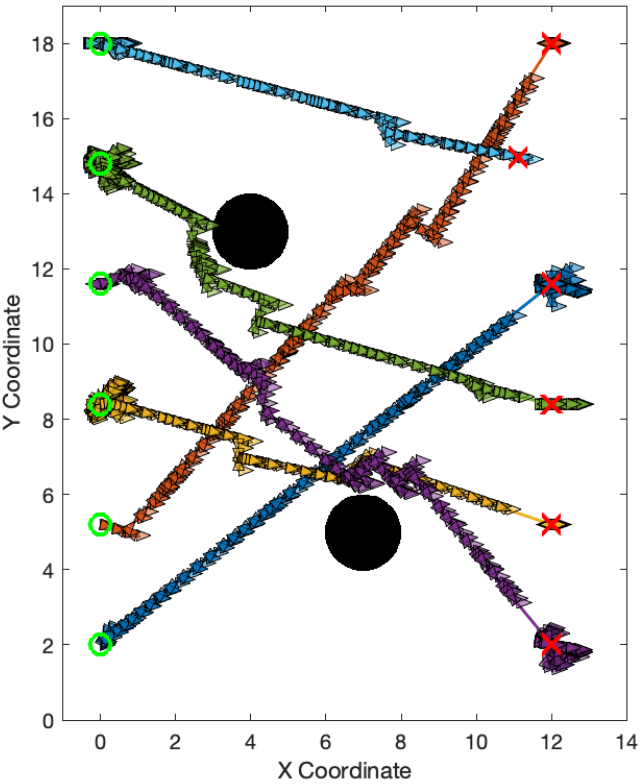} \hfill
        \includegraphics[width=0.32\textwidth]{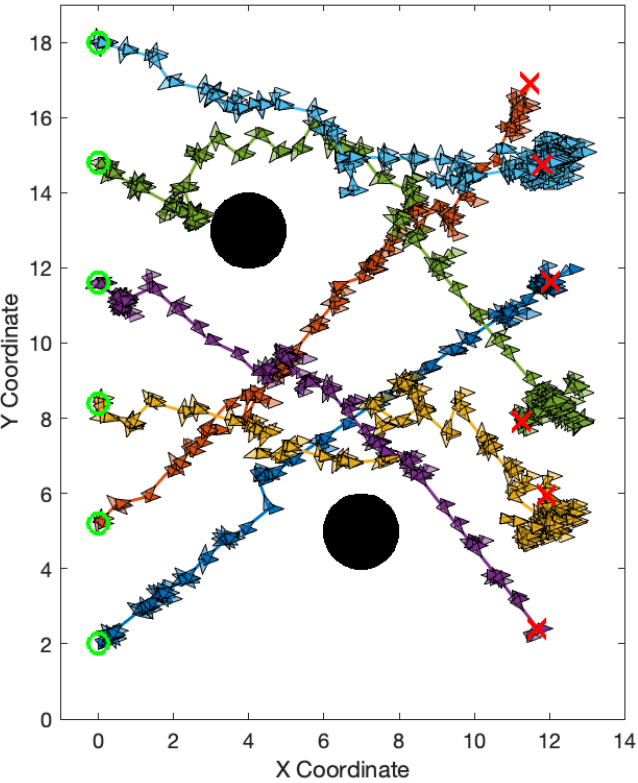} \hfill
        \includegraphics[width=0.32\textwidth]{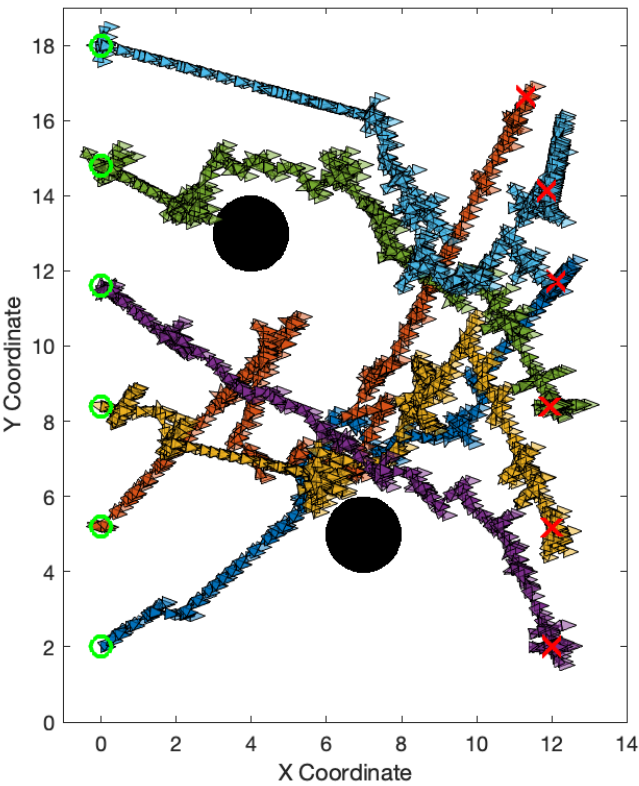} \\[-9mm]
\makebox[0.32\textwidth][c]{\scriptsize a (EC2)} \hfill
\makebox[0.32\textwidth][c]{\scriptsize b (EC1)} \hfill
\makebox[0.32\textwidth][c]{\scriptsize c (IC3)}       \\[6mm]
        \includegraphics[width=0.32\textwidth]{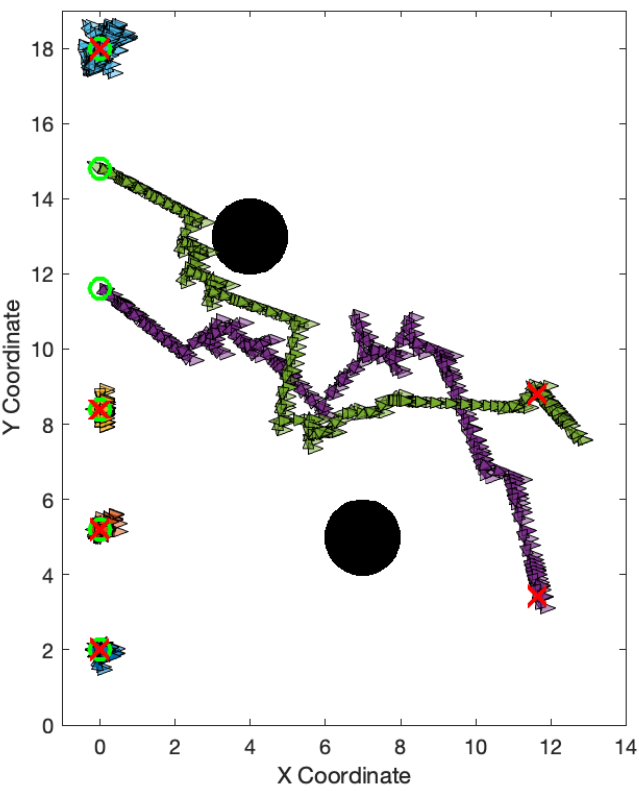} \hfill
        \includegraphics[width=0.32\textwidth]{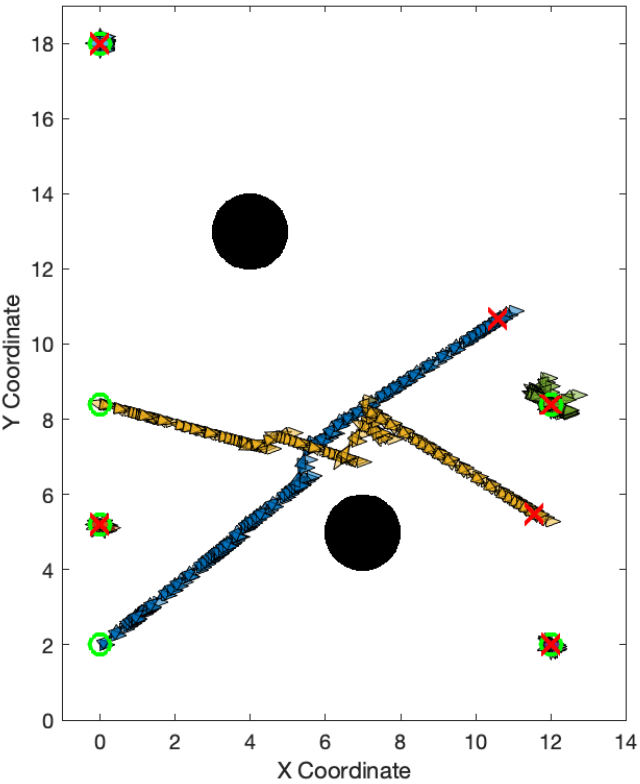} \hfill
        \includegraphics[width=0.32\textwidth]{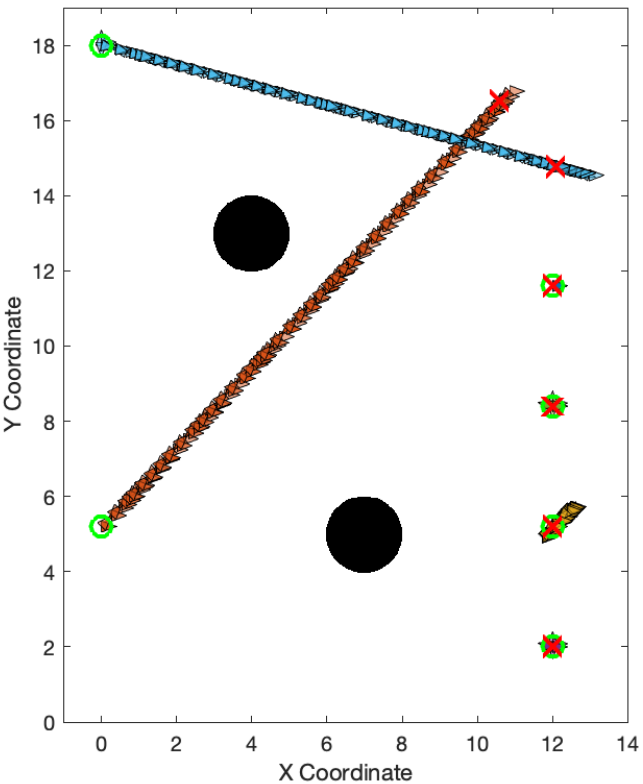}\\[-9mm]
\makebox[0.32\textwidth][c]{\scriptsize d (IC4 4, 5)} \hfill
\makebox[0.32\textwidth][c]{\scriptsize e (IC4 1, 3)} \hfill
\makebox[0.32\textwidth][c]{\scriptsize f (IC4 2, 6)}       \\
\end{center}
  \caption{(a) EC2: Pure control; (b) EC1: Pure planning using RRT with rotation; (c) IC3: Parallel motion; (d) IC4: Robots 4 and 5; (e) IC4: Robots 1 and 3; and (f) IC4: Robots 2 and 6.}
\label{exper:RRT6robot_obstable}
\vspace{-6mm}
\end{figure}

\begin{table}[h]
\begin{center}
    \begin{tabular}{|r|r|r|r|r|}
        \hline
        \textbf{Case} & \textbf{RRT Runtime} & \textbf{RRT Nodes}& \textbf{Path Length} & \textbf{Execution}  \\
        \hline
        RRT with rotation$^3$  & 1141.38s  & 8242.75 & 384.62 & 200.57s \\
        \hline
        Subgroup parallel & 216.15s  & 2218.01 &  868.66 & 886.04s \\
        \hline
        Subgroup sequential & 15.38s & 102.10  & 529.47  & 577.41s \\
        \hline
        Pure control & 19.44s & 175.50 & 871.82 & 1042.47s \\
        \hline
    \end{tabular}
    \end{center}
%    \vspace{-3mm}
     \caption{Performance of the four motion planning algorithms with obstacles (average of 10 instances). All algorithms terminate within a neighborhood of radius 1.5.  $^3$No solution was found for Original RRT in 8 out of 10 cases within the allotted time.}
       \label{Tab.RRTcases}
    \vspace{-8mm}
\end{table}

\subsubsection{Extreme Case 2 (EC2): Pure control}
In this scenario, each subgroup contains a single robot. The planning problem involves guiding individual robots to the goal. The primitives are of order $r=3$ and they move one robot at a time. Figure~\ref{exper:RRT6robot_obstable}a shows a sample path for 6 robots in 4 groups. The environment is the square $[0,20]\times[0,20]$ containing two black circular obstacles. The icon density is high on the path, and each repeated icon pointing in the same direction indicates the back-and-forth motion in primitives.
 % When we come to the robots size of $14, 30$  solving fmincon take a average time $46$ seconds sss second respectively and the path is no longer remains in the neighborhood of its initial positions, while as the nested primitives is pre hand-designed, it can solves in seconds. 

\subsubsection{Intermediate Case 3 (IC3): Parallel planning}

In this case, we explore the trade-off between the path execution time and RRT planning computation time. Figure~\ref{exper:RRT6robot_obstable}c presents an example in which we use $3$ disjoint subgroups, each containing $2$ different robots. These subgroups are inspired by the Lie brackets $[h_2, g_1]$, $[h_3, g_2]$ and $[h_1, g_3]$. This experiment was repeated $10$ times, with an average planning time of $216.15$s and an execution time of $886.04$s.

\subsubsection{Intermediate Case 4 (IC4): Sequential planning}
 
In contrast to IC3, here we plan the motion of each subgroup sequentially. This effectively reduces the dimensionality of planning problems.  Figures~\ref{exper:RRT6robot_obstable}d-f show the motions where 6 robots were divided into 3 subgroups using Lie brackets $[h_2,g_1]$, $[h_1,g_3]$ and $[h_3, _2]$. Compared with Fig.~\ref{exper:RRT6robot_obstable}c, the average planning time, path length, and execution time decrease to $19.44$s, $529.47$, and $577.41$s, respectively, under the same RRT parameters.

Table~\ref{Tab.RRTcases} compares the performance of different strategies for solving the motion planning problem in the presence of obstacles. RRT generates the most efficient path, but it takes considerably longer to find it. On the other hand, pure control is very efficient in finding the path, but the path is inefficient. It can be seen that the sequential plan provides an excellent compromise between the computation time and the efficiency of the generated path. For large swarms, sequential planning is computationally much more efficient than parallel planning.

%% file: 6_Conclusion.tex
\section{Conclusion}

This paper investigates how a novel group-control framework can be used for motion planning for microrobot swarms controlled by a global field. We focus on a swarm of MicroStressBots, electrostatically actuated MEMS stress microrobots. We show that the system is Small-Time Locally Controllable (STLC) in positions even though the control inputs are unilateral. We also demonstrate that the minimum number of groups to achieve STLC for $n$ robots is $log_2(n + 2) + 1$. We study the complexity of the motion planning problem for a MicroStressBot swarm under group-based control. In particular, we compare the trade-off between the complexity of control, the complexity of motion planning, and the efficiency of the generated path. We introduce the notion of motion primitives and subgroups, which allows us to balance these different factors. We show that subgroups significantly simplify the motion planning problem and make the approach applicable to larger swarm sizes. Simulations confirm the effectiveness of the framework, suggesting significant potential for applications requiring microrobot coordination, such as microassembly, tissue engineering, and drug delivery.
%Future work will explore more complex scenarios, larger swarm sizes, and the design of more efficient primitives.